\documentclass[a4paper,12pt]{article}
\usepackage{PRIMEarxiv}
\usepackage{amsfonts, amsmath, amssymb, amsthm}
\usepackage{algorithm}
\usepackage{algpseudocode}
\usepackage[english]{babel}
\usepackage{booktabs}
\usepackage{mathabx}
\usepackage{bm}
\usepackage{bbm}
\usepackage{graphicx}
\usepackage{cases}
\usepackage{eucal}
\usepackage{enumerate}
\usepackage[shortlabels]{enumitem}
\usepackage{float}
\usepackage{dsfont}
\usepackage{epsfig}
\usepackage{fontenc}
\usepackage{tabularx}
\usepackage[colorlinks, allcolors = blue]{hyperref}
\usepackage{lscape}
\usepackage{mathrsfs}
\usepackage{mdframed}
\usepackage{natbib}
\usepackage{rotating}
\usepackage{subfig}
\usepackage{verbatim}
\usepackage[table]{xcolor}
\usepackage{fancyhdr}
\usepackage{multirow}

\def\P\text{P}

\newcommand{\dif}{\ensuremath{\mathrm{d}}}

\newcommand{\T}{\ensuremath{\mathrm{\scriptscriptstyle T}}}
\usepackage{bm}

\newcommand{\Thetab}{\ensuremath{\bm\Theta}}

\newcommand{\Sigmab}{\ensuremath{\bm\Sigma}}

\newcommand{\Psib}{\ensuremath{\bm\Psi}}

\newcommand{\thetab}{\ensuremath{\bm\theta}}

\newcommand{\mub}{\ensuremath{\bm\mu}}

\newcommand{\psib}{\ensuremath{\bm\psi}}

\usepackage{amsmath}

\DeclareMathOperator{\E}{E}

\usepackage{mathrsfs}

\newlength\aftertitskip     \newlength\beforetitskip
\newlength\interauthorskip  \newlength\aftermaketitskip

\if@abbrvbib
\else
\fi

\bibpunct{(}{)}{;}{a}{,}{,}

\newtheorem{theorem}{Theorem}

\newtheorem{example}{Example}

\usepackage{xcolor}
\def\XXint#1#2#3{{\setbox0=\hbox{$#1{#2#3}{\int}$ }
\vcenter{\hbox{$#2#3$ }}\kern-.6\wd0}}

\title{A parallelizable model-based approach for marginal and multivariate clustering}

\author{
  Miguel de Carvalho\\
  School of Mathematics\\
  University of Edinburgh\\
  EH93FD, Edinburgh, UK\\
  \texttt{miguel.decarvalho@ed.ac.uk} \\
   \And
  Gabriel Martos\\
  Departamento de Matemática y Estadística\\
  Universidad Torcuato Di Tella\\
  Buenos Aires, Argentina\\
  \texttt{gmartos@utdt.edu} \\
   \And
  Andrej Svetlo\v{s}\'{a}k\\
  School of Mathematics\\
  University of Edinburgh\\
  EH93FD, Edinburgh, UK\\
  \texttt{andrej.svetlosak@ed.ac.uk} \\
}

\begin{document}
\maketitle

\begin{abstract}
  This paper develops a clustering method that takes advantage of the sturdiness of model-based clustering, while attempting to mitigate some of its pitfalls. First, we note that standard model-based clustering likely leads to the same number of clusters per margin, which seems a rather artificial assumption for a variety of datasets. We tackle this issue by specifying a finite mixture model per margin that allows each margin to have a different number of clusters, and then cluster the multivariate data using a strategy game-inspired algorithm to which we call Reign-and-Conquer. Second, since the proposed clustering approach only specifies a model for the margins---but leaves the joint unspecified---it has the advantage of being partially parallelizable; hence, the proposed approach is computationally appealing as well as more tractable for moderate to high dimensions than a `full' (joint) model-based clustering approach. A battery of numerical experiments on artificial data indicate an overall good performance of the proposed methods in a variety of scenarios, and real datasets are used to showcase their application in practice.
\end{abstract}


\section{Introduction}\label{introduction}
\noindent \textbf{Context and Motivation} \\
Clustering is an unsupervised learning approach for the task of partitioning data into meaningful subsets. The huge literature on cluster analysis is difficult to survey in a few sentences, but a concise description of well-known approaches is offered by \citet{friedman2009}, \citet{everitt2011}, and \citet{king2014}. Examples of mainstream methods for clustering data include model-based \citep{bouveyron2019}, similarity-based \citep{macqueen1967, kaufman1987}, and hierarchical clustering \citep[][Section~14.3]{friedman2009}
. In this paper we propose a novel model-based approach for cluster analysis that lies at the interface of model-based clustering (i.e.,~via mixture models) and similarity-based clustering (i.e.,~via $\mathcal{K}$-means and $\mathcal{K}$-medoids). The proposed approach aims to benefit from the flexibility and soundness of model-based clustering, while attempting to mitigate Pitfalls~1 and 2 below. Model-based clustering is a fast-evolving and intradisciplinary research topic as can be seen from the recent Handbook on Mixture Analysis \citep{fruhwirth2019} as well as the survey papers of \cite{melnykov2010}, \cite{mcnicholas2016}, \cite{gormley2023}, and the references therein.\\

\noindent \textbf{Pitfall~1: The ``Single $\mathcal{K}$ Problem''} \\ 
The idea of thinking of a cluster as a component of a mixture model has a long tradition in cluster analysis, that has its roots in Tiedeman's work in 1955 \citep{mcnicholas2016}. Despite the resilience and flexibility of this paradigm, it is often unnoticed that multivariate model-based clustering may induce the same number of clusters on each margin. For many applied contexts of interest it is however unnatural to believe that all margins should have exactly the same number of components---and hence the same number of marginal clusters. To appreciate this issue, let's revisit the Gaussian finite-mixture model,
\begin{equation}\label{mvn}
  f(\bm{x}) = \sum_{k = 1}^{\mathcal{K}} \pi_k \phi_d(\bm{x}; \mub_k, \Sigmab_k), \quad \bm{x} = (x_1, \dots, x_d), 
\end{equation}
where $\phi_d(\bm{x}; \mub_k, \Sigmab_k)$ is the density function of a $d$-dimensional multivariate Normal distribution with mean $\mub_{k} = (\mu_{k, 1}, \dots, \mu_{k, d})$ and variance-covariance matrix $\Sigmab_k$, with diagonal elements $(\sigma^{2}_{k, 1}, \dots, \sigma^{2}_{k, d})$. The marginal distributions stemming from \eqref{mvn} are 
\begin{equation}\label{margins}
  f_j(y) = \sum_{k = 1}^{\mathcal{K}} \pi_k \phi(y; \mu_{k, j}, \sigma^2_{k, j}),
\end{equation}
for $j = 1, \dots, d$. As can be seen from \eqref{margins} model-based clustering as in \eqref{mvn} implies that all margins have $\mathcal{K}$ clusters per margin, except if $\mu_{k, j} = \mu_{k', j}$ and $\sigma_{k, j} = \sigma_{k', j}$ for some $k'$ and $k$. Since in practice it is challenging to learn from data if this (i.e., $\mu_{k, j} = \mu_{k', j}$ and $\sigma_{k, j} = \sigma_{k', j}$) holds exactly, we will refer to this challenge as the ``single $\mathcal{K}$ problem.''  \\

\noindent \textbf{\noindent \textbf{Pitfall~2: Curse of Dimensionality, with $O(d^2)$ Parameters as $d \to \infty$}} \\
The Gaussian mixture model in \eqref{mvn} has $(\mathcal{K} - 1) + \mathcal{K}d + \mathcal{K}d(d + 1) / 2$ parameters, and hence the number of parameters increases quadratically with $d$. This shortcoming is well known to limit the scope of application of model-based clustering on high-dimensional data \citep{bouveyron2014}. Some approaches have been developed with the aim of providing a more parsimonious specification, and hence as byproduct this paper will also contribute to that literature. A key paper on parsimonious model-based clustering is that of \cite{mcnicholas2008} who suggest a latent Gaussian model that can be regarded as a mixture of factor models. \vspace{-0.1cm}\\ 

\noindent \textbf{Summary of Main Contributions} \\ 
The main contributions of this paper are as follows: 
\begin{itemize}
\item We pioneer the development of a model-based solution for the ``single $\mathcal{K}$ problem'' outlined in \eqref{margins}, by specifying an individual finite mixture models for each of the margins, but making no assumptions on the joint distribution. The sample space is then partitioned via a strategy game-inspired algorithm, which can be used for clustering data, both marginally as well as in a multivariate fashion.
\item We develop a computationally appealing and partially parallelizable model-based approach that bypasses the need to learn about $\mathcal{K}d(d+1)/2$ parameters used in the covariance matrices $\boldsymbol\Sigma_1, \dots, \boldsymbol\Sigma_{\mathcal{K}}$ required for a `full' (joint) Gaussian model-based clustering approach.
\item The proposed data-driven approach for partitioning the sample space, automatically sieves regions that only have a residual amount of mass---via a  minimum entry-level requirement that is specified by the user or set in a data-driven manner. In addition, we assess numerically the proposed methodologies and ascertain the reliability of their clustering performance in a battery of numerical experiments.
\item As a byproduct, the proposed method contributes to the literature on game-inspired clustering approaches that followed from the seminal paper of \cite{bulo2009} \citep[e.g.,][]{hou2022}. As will be shown below the proposed approach differs however significantly from that of the previous paper---both in terms of scope (the focus of \citeauthor{bulo2009} is on hypergraph clustering) as well as on the specificities of the game underlying the proposed clustering approach.
\end{itemize}
\noindent \textbf{Structure and Organization of this Paper}\\
The remainder of this paper unfolds as follows. In Section~\ref{model} we introduce the probabilistic framework underlying the partition of the sample space which will be the building block of the proposed clustering approach to be introduced in  Section~\ref{sample}. Section~\ref{gtheory} outlines a conceptualization of a variant of the proposed partitioning approach by reinterpreting it as a strategy game. Experiments with artificial and real data are conducted in Sections~\ref{numerical} and \ref{applications}, respectively. Final observations closing remarks are given in Section~\ref{closing}.

\section{{Reign-and-Conquer Partitioning}}\label{model}

\subsection{\large{{The Probabilistic Framework}}}\label{marginal}
A key goal in this section is
to devise a partition of the sample space of the joint distribution that
is meaningful in a sense to be made more clear below.
The proposed framework entails three steps, and to streamline the presentation we
first focus on the bivariate setting. Comments on the mutivariate
extension are given in Section~\ref{multivariate}, and
Section~\ref{gtheory} outlines a game-theoretical variant of the
proposed approach. This section does not yet consider data nor
estimation, it rather focuses on a probabilistic setup for
partitioning a sample space; comments on learning from data based on
the principles below are given in Section~\ref{sample}. Here and below, no
assumption whatsoever is made on the joint density, and we model each
margin as a mixture model.  Keeping in mind that any density can be
approximated by a mixture of Normals, given enough components, the
latter assumption is relatively mild.

\subsubsection*{Step~1: Margins \strut \hfill \footnotesize (Model-Based Clustering)}
Let $X \sim f_{X}$ and $Y \sim f_{Y}$, where
\begin{equation}\label{mix}
  f_{X}(x \mid K, \Thetab) = \sum_{k = 1}^{K} \pi_k \, p(x \mid \thetab_k), \qquad 
  f_{Y}(y \mid L, \Psib) = \sum_{l = 1}^{L} \omega_l \, q(y \mid \psib_l). 
\end{equation}
Here, $p$ and $q$ are density functions, with parameters $\Thetab = (\thetab_1, \dots, \thetab_K)$ and $\Psib = (\psib_1, \dots, \psib_L)$; in addition $K$ and $L$ are the number of clusters respectively associated with the margins $X$ and $Y$.

\subsubsection*{Step~2: Reign \strut \hfill \footnotesize (Similarity-Based Joint Protocluster Allocation)}
We first divide the sample space of $(X, Y)$, to be denoted by $\Omega$, via a partition that is
based on the set of all marginal cluster means 
\begin{equation}\label{cluster}
  \begin{cases}
    \begin{split}
      \mu_{X} = \{\mu_{X}^{(1)}, \dots, \mu_{X}^{(K)}\},  \\ 
      \mu_{Y} = \{\mu_{Y}^{(1)}, \dots, \mu_{Y}^{(L)}\},      \\ 
    \end{split}   
  \end{cases}
  \qquad \text{where} \qquad 
  \begin{cases}
    \begin{split}
      \mu_{X}^{(i)} = E(X \mid \thetab_i) = \int x \, p(x \mid \thetab_i) \, \dif x,  \\ 
      \mu_{Y}^{(i)} = E(Y \mid \psib_i) = \int y \, q(y \mid \psib_i) \, \dif y. \\ 
    \end{split}       
  \end{cases}
\end{equation}
Specifically, to each point $(\mu_{X}^{(i)}, \mu_{X}^{(j)})$ in the Cartesian product 
\begin{equation}\label{product}
  \mu_X \times \mu_Y  = \{(\mu_{X}^{(1)}, \mu_{Y}^{(1)}), \dots, (\mu_{X}^{(K)}, \mu_{X}^{(L)})\}, 
\end{equation}
corresponds a Voronoi cell $A_{i, j}$ for $i = 1, \dots, K$ and $j = 1, \dots, L$. We refer to the Voronoi cells $A_{1,1} ,\dots, A_{K,L}$ as \emph{protoclusters} as they define a first partition of $\Omega$, and call the sites, $\mu_X \times \mu_Y$, as  \emph{protocluster centers}.

\subsubsection*{Step~3: Conquer \strut \hfill \footnotesize (Final Joint Cluster Allocation)} \label{Conquer}
After dividing $\Omega$ we conquer. That is, Step~3 identifies low density protoclusters to be conquered by high density regions, hence refining the naive partition of $\Omega$ from Step~2. To avoid including in the resulting partition regions that have a residual amount of mass, a minimum entry-level requirement is chosen to which we refer to as the sieve size $u \in [0,1]$. Let
\begin{equation}
  \label{Du}
  D_u \equiv \{(i,j): P(A_{i,j})\leq u\},
\end{equation}
be the indices of the protoclusters that have low mass and that hence will be conquered for a given sieve size. The final sample space partition corresponds to the Voronoi cells $B_{i,j}$ associated with the protocluster centers of the conquerors, i.e., $(\mu^{(i)}_X,\mu^{(j)}_Y)$ with $(i,j) \in D_u^c$. To assess how the number of final clusters depends on the sieve size, we define the \textit{conquering function} as
\begin{equation}
  \label{eq:conq}
  C(u)=|D_{u}^c| = KL - |D_u|,
\end{equation}
where $|A|$ denotes the cardinality and $A^c$ is the complement of the set $A$. 

Example~\ref{mix3} illustrates the main concepts and ideas of the sample space partitioning approach discussed above.

\begin{figure}[h]
  \centering 
  \begin{tabular}{cc}
    \subfloat[]
    {\includegraphics[scale=0.24]{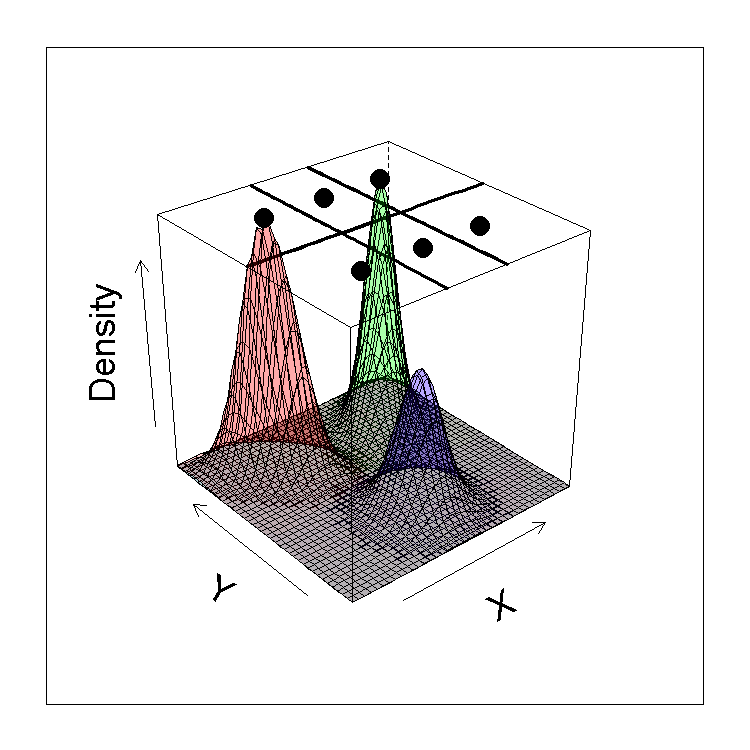}}
    \subfloat[]
    {\includegraphics[scale=0.3]{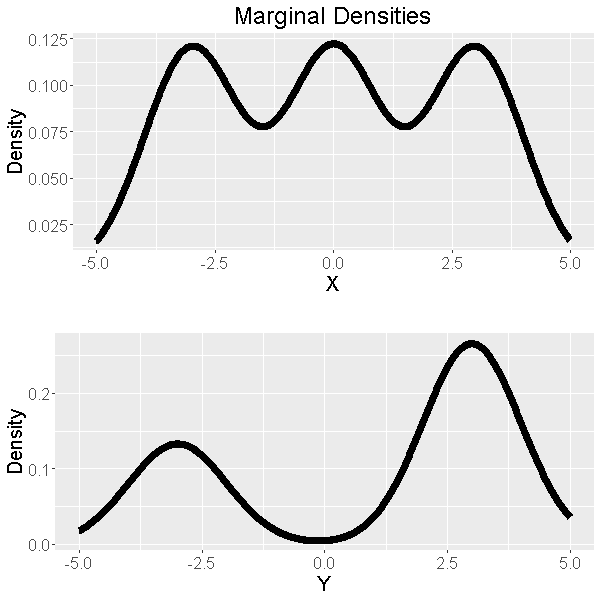}}
  \end{tabular}
  \centering \\
  \centering 
    \begin{tabular}{cc}
    \subfloat[]
    {\includegraphics[scale=0.3]{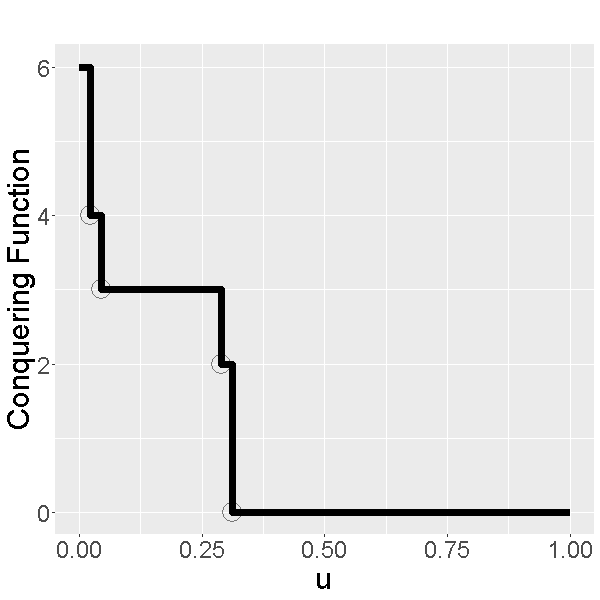}} \hspace{0.19cm}
    \subfloat[]
    {\includegraphics[scale=0.24]{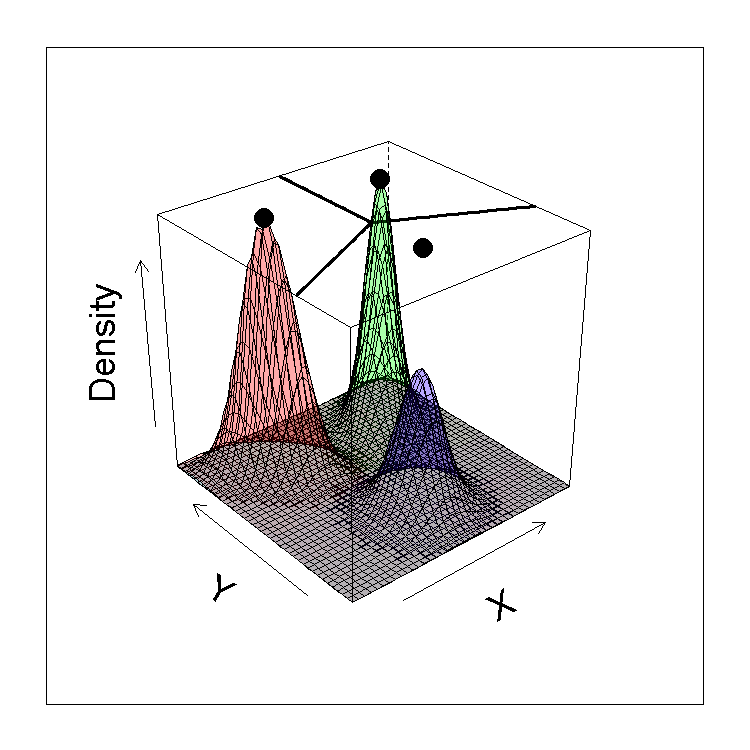}}
\end{tabular}
\caption{\footnotesize Reign-and-Conquer Partitioning for Example~\ref{mix3}. (a) protoclusters, protocluster centers ($\bullet$), and joint density. (b) Marginal densities. (c) Conquering function. (e) The Voronoi cells of the conquerors for $u=0.1$.\label{fig:ex1}}
\end{figure}

\begin{example}[Reign-and-conquer partitioning on a mixture of 3 bivariate Normal distributions]\label{mix3}\normalfont
In Figure~\ref{fig:ex1}~(a) we depict a mixture of 3 bivariate Normal distributions as in Equation~\eqref{mvn}, with means $\boldsymbol \mu_1 =(-3,3)$, $\boldsymbol \mu_2 =(3,3)$,   $\boldsymbol \mu_3 =(0,-3)$, mixing probabilities  $\pi_1=\pi_2 = \pi_3 = 1/3$, and covariance matrices:
$$ \boldsymbol \Sigma_1 =   \begin{bmatrix}
1&0.5\\
0.5&1
  \end{bmatrix},\quad  \boldsymbol \Sigma_2 =   \begin{bmatrix}
1&-0.5\\
-0.5&1
  \end{bmatrix},\quad \boldsymbol \Sigma_3 =   \begin{bmatrix}
1&0\\
0&1
  \end{bmatrix}.$$ 
This set-up leads to a different number of clusters per margin as can be seen in Figure~\ref{fig:ex1}~(a) and (b). Specifically, in the $Y$--margin there are two clusters with centers in $\mu_{Y}^{(1)}=3$ and $\mu_{Y}^{(2)}=-3$ (i.e., $L = 2$), while in the $X$--margin there are 3 clusters with centers in $\mu_{X}^{(1)}=3$, $\mu_{X}^{(2)}=0$ and $\mu_{X}^{(3)}=-3$ (i.e., $K = 3$). 
In Figure~\ref{fig:ex1} (c), we also display the conquering function, and in Figure~\ref{fig:ex1} (d) we depict the Voronoi cells of the conquerors corresponding to the sieve size $u=0.1$. If we were to regard protoclusters from Step~2 as `territories,' what the Reign-and-Conquer Partition does at Step~3 is to conquer low density cells and lets the dominant regions of mass conquer them. 
\end{example}

\subsection{\large{{$d$-Dimensional Extension and Theoretical Properties}}}\label{multivariate}
The approach from Section~\ref{marginal} extends naturally to a $D$-dimensional context as follows. For the margins, we now consider $X_1 \sim f_1, \dots, X_d \sim f_d$ with
\begin{equation}\label{mix2}
  f_j(x \mid K_j, \Thetab_j) = \sum_{k = 1}^{K_j} \pi_{k, j} \, p_j(x \mid \thetab_k), 
\end{equation}
where the notation in \eqref{mix2} extends that in \eqref{mix}, with $j = 1, \dots, d$. In particular, \eqref{mix2} implies that the first margin ($X_1$) has $K_1$ clusters, that the second margin has $K_2$ clusters, and so on. The partition underlying the divide step is  now formed by the Voronoi tesselation $\{A_{\mathbf{i}}: \mathbf{i} \in I\}$, with $\mathbf{i} = (i_1, \dots, i_d)$, $I = \{1, \dots, K_1\} \times \cdots \times \{1, \dots, K_d\}$, and where the $A_{\mathbf{i}}$ cell corresponds to the protocluster center $(\mu_{X_1}^{(i_1)}, \dots, \mu_{X_d}^{(i_d)})$, with $\mu_{X_j}^{(k)} = \E(X_j\,|\,\thetab_k) = \int x \, p_j(x \,|\, \thetab_k) \, \dif x$. The 
final clusters yield from the conquering step correspond to the Voronoi cell $B_{\mathbf{i}}$ associated with the protocluster centers of the conquerors, i.e., $(\mu^{(i_1)}_{X_1},\dots, \mu^{(i_d)}_{X_d})$ with $\mathbf{i} \in D_u^c$, where
\begin{equation}\label{Du2}
  D_u \equiv \{\mathbf{i} \in I: P(A_{\mathbf{i}})\leq u\},
\end{equation}
for $u \in [0, 1]$. The conquering function is more generally defined in the $d$-dimensional setting as 
\begin{equation}
  \label{conq2}
  C(u)=|D_{u}^c| = K_1 \times \cdots \times K_d - |D_u|.
\end{equation}
The conquering function is characterized by a set of properties summarized in the next theorem. 
\begin{theorem}\label{propsC}
The conquering function, $C(u)$ as defined in \eqref{conq2}, obeys the following properties:
\begin{enumerate}
\item It is nonincreasing.
\item It is continuous from the left.
\item It is bounded below by $C(1)=0$ and above by $C(0) = \prod_{j = 1}^{d} K_j$.
\item It integrates to one, i.e., $\int_{0}^{1} C(u) \, \dif u =1$.
\end{enumerate}  
\end{theorem}
\begin{proof}
  See Appendix~A.
\end{proof}
\noindent It can be noticed that the conquering function from Example~\ref{mix3} verifies all the claims of Theorem~\ref{propsC} as can be seen from Figure~\ref{fig:ex1} (c). In addition to Theorem~\ref{propsC} it can be shown that the conquering function is a step function that has a finite number of steps provided that $K_1, \dots, K_d$ are finite. See Appendix~B. 

\section{{Learning from Data}}\label{sample}


We now devise an algorithm based on the probabilistic framework from Section~\ref{model}. The goal is to allocate observations in a dataset $\{\mathbf{x}_l\}_{l = 1}^n$, with $\mathbf{x}_l = (x_{l, 1}, \dots, x_{l, d})^{\T}$ in $\mathbb{R}^d$, into a set of meaningful classes---both in terms of the margins as well as the joint. Using the notation from  Section~\ref{multivariate}, we  introduce the R2C (Reign-and-Conquer Clustering) algorithm.

\begin{algorithm}
\caption{R2C (Reign-and-Conquer Clustering)\label{ddc}}
\begin{algorithmic}
\item[~~~Step~1.~Margins] Fit the $j$th marginal density in \eqref{mix2} using $\{x_{1, j}, \dots, x_{n, j}\}$, for $j = 1, \dots, d$, so to learn about $\{(K_j, \mu_{X_j}^{(1)}, \dots, \mu_{X_j}^{(K_j)})\}_{j = 1}^d.$
\item[~~~Step~2.~Reign:] Learn about the protoclusters $\{A_{\mathbf{i}}: \mathbf{i} \in I\}$  of $\{(\mu_{X_1}^{(i_1)}, \dots, \mu_{X_d}^{(i_d)}): \mathbf{i} \in I\}$.
\item[~~~Step~3.~Conquer:] Learn about the Voronoi cell of the conquerors,   $\{B_{\mathbf{i}}:\mathbf{i} \in D_u^c\}$, that is that of $\{(\mu^{(i_1)}_{X_1},\dots, \mu^{(i_d)}_{X_d}): \mathbf{i} \in D_u^c\}$, and allocate the $l$th observation to a cluster using the encoder
    \begin{equation}\label{encoder}
      \textnormal{Enc} (l) =
      \operatorname{arg\,min}_{(i_1, \dots, i_d)} \| \mathbf{x}_l - (\mu^{(i_1)}_{X_1},\dots, \mu^{(i_d)}_{X_d})\|^2.     
    \end{equation} 
\end{algorithmic}
\end{algorithm}

\noindent If we were to regard protoclusters from Step~2 as `territories,' what R2C does at Step~3 is to let the mass dominant regions conquer the low density ones. The R2C algorithm warrants some further comments: 
\begin{itemize}
\item \textbf{Step 1.~Margins}: To learn about $\{(K_j,  \mu_{X_j}^{(1)}, \dots,  \mu_{X_j}^{(K_j)})\}_{j = 1}^d$ several approaches can be taken. We use the NLP (non-local prior) for mixtures approach of \cite{fuquene2019choosing}, but alternatively one could use, for example, RJ MCMC (Reversible Jump Markov Chain Monte Carlo) \citep{green1995}. We opt for NPLs as they are designed to enforce parsimony by penalizing mixtures with a redundant number of components, and they bypass the need for complicated algorithms such as RJ MCMC.   Determining $K_j$ is a well-studied yet open problem, and an overview of the literature in this can be found in \cite{richardson1997}, \cite{fraley2002}, and \cite{baudry2010}.
\item \textbf{Step~2.~Reign}: To compute the protoclusters $\{A_{\mathbf{i}}: \mathbf{i} \in I\}$  of $\{(\mu_{X_1}^{(i_1)}, \dots,  \mu_{X_d}^{(i_d)}): \mathbf{i} \in I\}$, we resort to the cluster centers from Step~1.
\item \textbf{Step~3.~Conquer}: To learn about the cells of the conquerors, we need to learn about $D_u = \{\mathbf{i} \in I:  P(A_{\mathbf{i}})\leq u\}$---and this implies estimating $P(A_{\mathbf{i}})$. Several approaches can be taken, and here we opt for the simplest one---the maximum likelihood estimator (MLE). To avoid burdening the notation, we introduce the MLE on the bivariate case, but the details extend easily to the multivariate setting. Let $\{\mathbf{x}_l\}_{l = 1}^n = \{(x_{l, 1}, x_{l, 2})\}_{l = 1}^n$ and note that the number of points falling on the protoclusters, $n_{i, j} = |\{\mathbf{x}_l \in A_{i, j}\}_{l = 1}^n|$  is Multinomial distributed, that is,
  \begin{equation}
    \label{eq:nmulti}
    \mathbf{n} \sim
    \text{Multinomial}(\mathbf{p}),
  \end{equation}
  where $\mathbf{n} = (n_{1, 1}, \dots, n_{K, 1}, \dots, n_{1, L}, \dots, n_{K, L})$ and $\mathbf{p} = ({p}_{1, 1}, \dots, {p}_{K, 1}, \dots, {p}_{1, L}, \dots, {p}_{K, L})$. 
  Hence, the MLE is $\hat{\mathbf{p}} = \mathbf{n} / n$ and Bayesian inference can also be easily conducted.\footnote{Bayesian inference can be conducted by assuming a Dirichlet prior over the unit simplex on $\mathbb{R}^{KL}$, i.e., $\mathbf{p} \sim \text{Dirichlet}(\mathbf{a}),$ where $\mathbf{a} = (a_{1, 1}, \dots, a_{K, 1}, \dots, a_{1, L}, \dots, a_{K, L})$. Dirichlet--Multinomial conjugacy then implies that posterior inferences can be obtained from $\mathbf{p} \mid \mathbf{n} \sim \text{Dirichlet}(\mathbf{a} + \mathbf{n})$. Finally, another alternative would be to specificy a model for the joint distribution. Yet, given that we only need to learn about the ${p}_{i, j} = P(A_{i, j})$, and since we prefer to avoid specifying a copula that may not accurately describe the joint distribution, we opt for the above-described likelihood-based approaches.} Following the principles from Section~\ref{model}, this estimate implies conquering protoclusters centered at $({\mu}_k, {\mu}_l)$, for which 
\begin{equation}\label{criterion}
  \frac{n_{k,l}}{n} < u, \quad \text{for } u \in (0,1]. 
\end{equation}
The estimated regions of the conquerors ${B}_{k,l}$ are obtained by the Voronoi tessellation on the remaining $({\mu}_k, {\mu}_l)$ so that $(k,l) \notin {D}_u \equiv \{(i,j): P({A}_{i,j}<u)\}$. 
\end{itemize}
To a certain extent, the R2C algorithm combines the paradigms of model-based clustering and similarity-based clustering. Indeed, Step~1 consists of a marginal model-based clustering approach. In addition, just as in similarity-based clustering methods, such as $k$-means, Step~3 entails an encoder \citep[][Section~14.3]{friedman2009}, which determines to which cluster observation $\mathbf{x}_l$ belongs to. The Euclidean norm in~\eqref{encoder} can be replaced by any preferred norm. For example using the Mahalanobis norm would take also the spread of the cluster into account, additionally to the distance from the cluster center. 

Some comments on the implementation and computing are in order.
In terms of implementation, as mentioned earlier, to avoid including in the resulting partition of the sample space regions that have a residual amount of mass, a minimum entry-level requirement $u \in [0,1]$ should be set by the user. That value might be set at a fixed low level (say, $u = 0.1$), so that all resulting clusters have at least that mass. Alternatively, data-driven approaches for setting $u$, based on the fitted conquering function, are also explored in Section~\ref{numerical}.
In terms of computing, Step~1 can be parallelized into $D$ cores, and so to speed up the
computations parallel computing was implemented with the R package 
\texttt{parallel} \citep{rdevelopmentcoreteam2022}. Steps~2 and 3 involve the computation of Voronoi
tesselations from out of $a \equiv |\mu_{X_1} \times \dots \times
\mu_{X_D}| = K_1 \times \dots \times K_D$ protocluster centers and
from $a - |D_u|$ protocluster centers of the conquerors. While
parallel algorithms could have been employed also for higher-performance
computation of Steps 2 and 3 \citep[e.g.,][]{peterka2014} we have
opted for the simple partially parallelizable approach in Algorithm~\ref{ddc}.

We close this section with a simple yet important comment. While
Step~3 of the R2C algorithm leads to multivariate clustering of $\{\mathbf{x}_l\}_{l = 1}^n$, marginal
clustering can be made directly from Step~1 via the posterior probabilities
\begin{equation}\label{estZ}
  \hat Z_{l, k, j} = \frac{\pi_{k, j} p_j(x_{l, j} | \boldsymbol{\theta}_k)}{\sum_{k = 1}^{K_j} \pi_{k, j} p_j(x_{l, j} | \boldsymbol{\theta}_k)}.  
\end{equation}
Note that $\hat Z_{l, k, j} \in [0, 1]$ estimates the cluster membership labels of the $l$th observation on the $j$th margin, which are defined as $Z_{l, k, j} = 1$ if the $l$th observation on the $j$th margin $x_{l, j}$ belongs to the $k$th component, or $Z_{l, k, j} = 0$ otherwise.

\section{An Outline of a Game Theory Conceptualization}\label{gtheory}
\subsection{{A Game of Thrones---Starting Point}}
This section outlines an alternative way to look into the sample space partitioning approach from Section~\ref{model} as a game. To streamline the presentation we focus on the bivariate case; the extension to the multivariate case is a matter of adjusting notation. More specifically, the game to be considered starts at Step~2 of Reign-and-Conquer Partitioning (Section~\ref{model}), players are to be understood as $KL$ `Kings' owning the protocluster `territories' ($\{A_{i, j}\}$) and who decide whether or not they will attack their neighbors. To make matters concrete, think of Figure~\ref{fig:ex1} (a) as representing the protocluster `territories' of $KL$ = 6 Kings, who have to decide whether or not they attempt to conquer the territories of their neighbors. If a territory is attacked by two Kings, they might have to share the conquered territory.

The neighboring structure of players can be represented via a $KL\times KL$ adjacency matrix $\textbf{M}$, and it can be visualized using a (undirected) graph $\mathscr{G} = (N, \mathscr{E})$, where $\mathscr{E}$ is a set of edges representing a neighboring relation. The outcome of the game is an element in $S$(to be defined in Section~\ref{thrones}),  and it can be visualized with a directed graph $\mathcal{G} = (N, \mathcal{E})$, where $\mathcal{E}$ is a set of directed edges or arrows representing attacks.
To build intuition surrounding these ideas and concepts, let's revisit Example~\ref{mix3}. Figure~\ref{rev} (a) depicts the graph corresponding to the neighboring structure of the $KL$ = 6 Kings, and the corresponding adjacency matrix is
  \begin{equation*}
    \textbf{M} = \begin{pmatrix}
      0 & 1 & 0 & 1 & 1 & 0 \\
      1 & 0 & 1 & 1 & 1 & 1 \\
      0 & 1 & 0 & 0 & 1 & 1 \\
      1 & 1 & 0 & 0 & 1 & 1      \\
      1 & 1 & 1 & 1 & 0 & 1       \\
      0 & 1 & 1 & 1 & 1 & 0     \end{pmatrix}.
\end{equation*}
    The directed graph in Figure~\ref{rev} (b) depicts an example of attack decisions compatible with the outcome from  Figure~\ref{fig:ex1} (d). Indeed, for example, we can think of the outcome in Figure~\ref{fig:ex1} (d) as the consequence of     players $(1,1)$ and $(2,2)$ attacking player $(2, 1)$ and sharing the conquered territory, and so on. 

\subsection{{Representation, Equilibrium, and Open Challenges}}\label{thrones}
Below, an `attack' is denoted with a `1', and `not to attack' with a `0'. The (normal form) game of interest is given by the triple $G = (N, \{S_{i}\}_{i \in N} , \{U_i\}_{i \in N})$, where:
\begin{itemize}
\item $N = \{(i, j): i = 1, \dots, K,\; j = 1, \dots, L\}$ is the set of players (`Kings').
\item $S_{i, j}$ is the pure set of strategies of King $(i, j)$, 
  $$S_{i, j} = \{\text{who to attack, keeping in mind that only neighbors can be attacked}\} \subseteq \{0, 1\}^{KL},$$
  and $S = \underset{(i, j) \in N}{\bigtimes} S_{i, j}$ is the set of all vectors of strategies, where `$\bigtimes$' is the Cartesian product.
\item $U_{i, j}(\mathbf{s}) $ is the payoff of King $(i, j)$, with $\mathbf{s} = (\mathbf{s}_{i, j})_{{(i, j)} \in N}$, with $U_{i, j}:S \to \mathbb{R}$. 
\end{itemize}
 
\begin{figure}
\begin{center}
  \includegraphics[scale = .2]{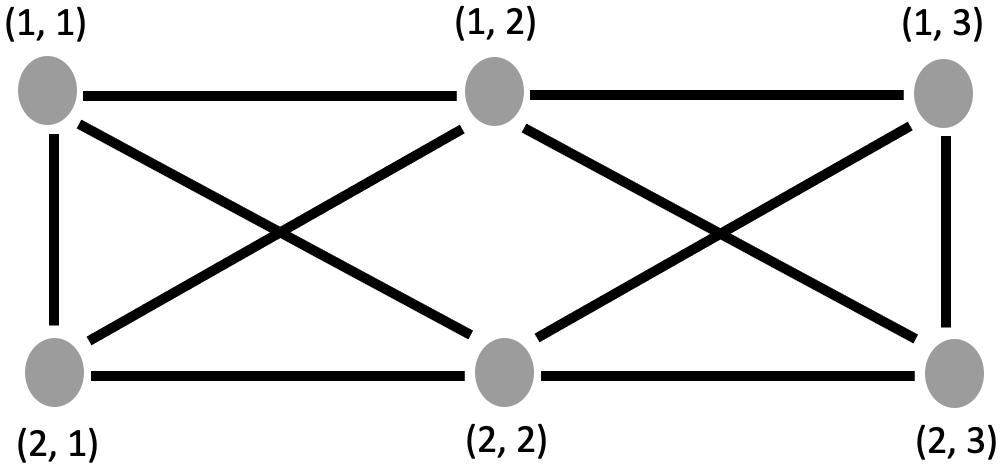} \hspace{0.75cm}
  \includegraphics[scale = .2]{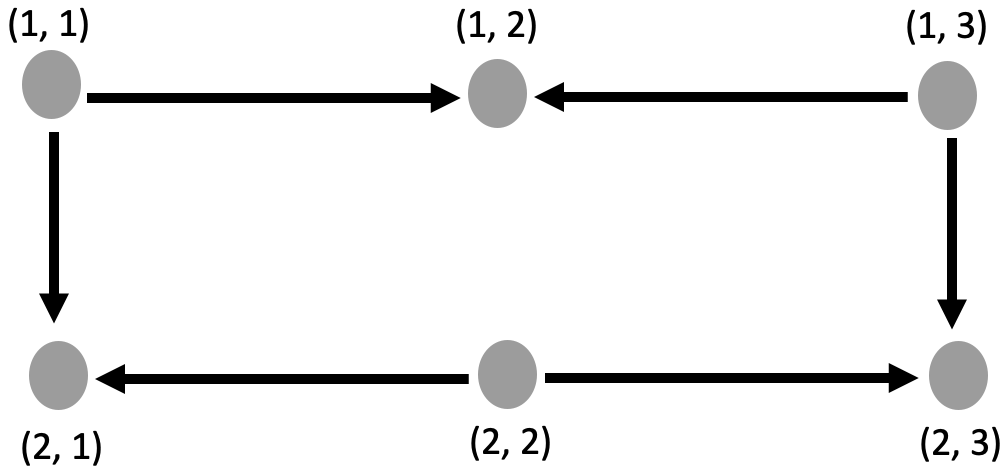}
\end{center}
  \hspace{4cm} (a)   \hspace{7.4cm} (b) 
  \caption{\label{rev} \footnotesize Revisiting Example~\ref{mix3}. (a) Neighboring structure corresponding to Figure~\ref{fig:ex1} (a). In (b) is depicted a directed graph with an instance of attack decisions compatible with the outcome from Figure~\ref{fig:ex1} (d). In both charts the nodes represent players (`Kings') (1, 1), \dots, (2, 3).}
\end{figure}

\noindent By construction, the strategy set of each player is finite and hence this is said to be a finite game. While a Nash equilibrium for this game may not exist over pure strategies, an equilibrium will exist over mixed strategies. A mixed strategy for player $(i, j)$ is a distribution over their set of pure strategies $S_{i, j}$, that is
\begin{equation*}
  \mathscr{S}_{i, j} = \bigg\{\sigma_{i, j}: S_{i, j} \to [0, 1] : \sum_{\mathbf{s}_{i, j} \in S_{i, j}} \sigma_{i, j}(\mathbf{s}_{i, j}) = 1\bigg\}.
\end{equation*}
The celebrated Nash theorem, recalled below for completeness, ensures that the game of interest has at least one equilibrium in mixed strategies. 
\begin{theorem}
Every finite game in strategic form, that has a finite number of players, has a Nash equilibrium in mixed strategies. 
\end{theorem}
\textbf{Proof}:
See, for example, \citet[][Section~5]{maschler2020}.\\

\noindent Conceptually speaking, the approach above endows Step~3 with a much broader range of possibilities on how to partition the sample space $\Omega$. First, there are numerous ways in which the `incentives' (utility functions) can be set, and in particular they can mimic the ones from Section~\ref{model}. A refinement of Step~3 based on the principles outlined above is as follows: a) Compute a Nash equilibrium; b) Derive the cells of the conquerors resulting from such equilibrium. In terms of a) we note that computation of Nash equilibria is nontrivial in general, but it can be conducted using simplicial subdivision \citep{van1987}, a Newton method known as Govindan--Wilson algorithm \citep{govindan2003}, search methods \citep{porter2008}, among other.

Keeping in mind the computational motivation of the paper, in the numerical experiments to be reported below we focus on  the computationally appealing approach from Section~\ref{sample}---that bypass the need for computing Nash equilibria in Step~3---but we aim to revisit the numerical performance of this game-theoretical variant of the proposed methods in future research. 

\begin{figure}[h!]
  \centering
  \footnotesize \textbf{Scenario~1}\\
  \begin{tabular}{cc}
\subfloat[]{\includegraphics[scale=0.3]{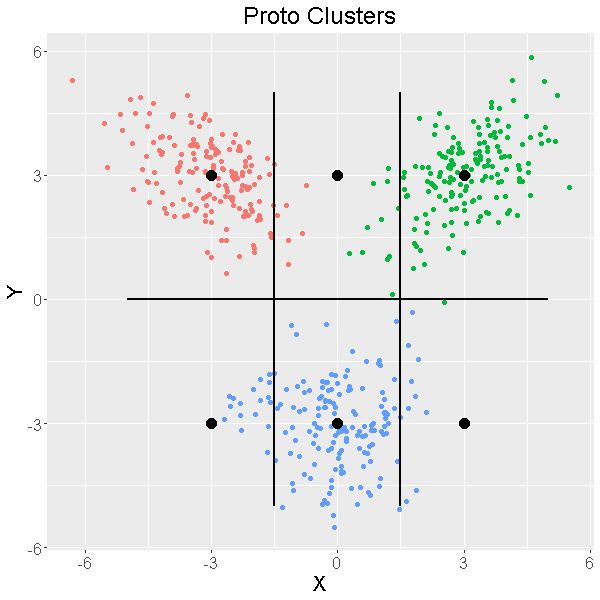}
} &
\subfloat[]{\includegraphics[scale=0.3]{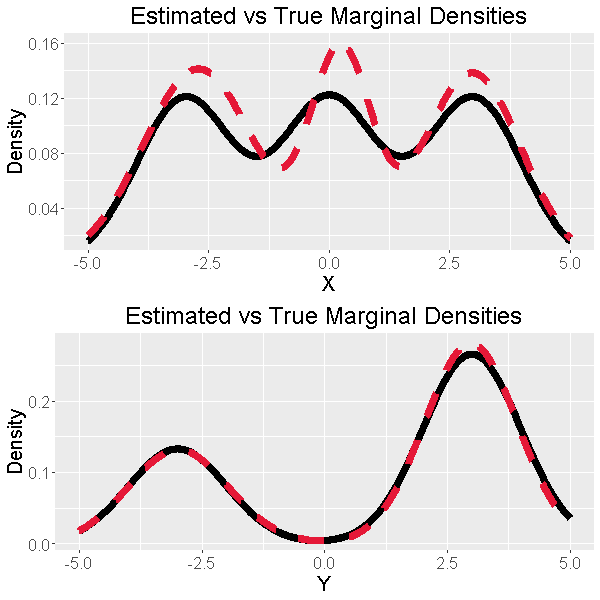}} 
\end{tabular}
\centering \\
\centering
  \begin{tabular}{cc}
\subfloat[]{\includegraphics[scale=0.3]{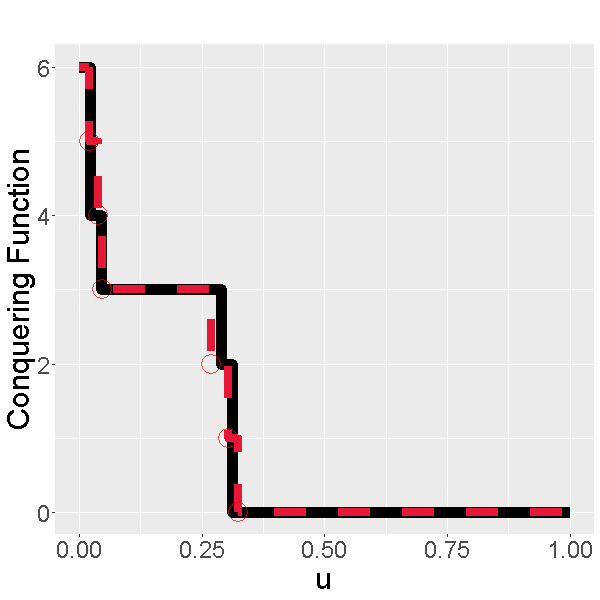}
} &
\subfloat[]{\includegraphics[scale=0.3]{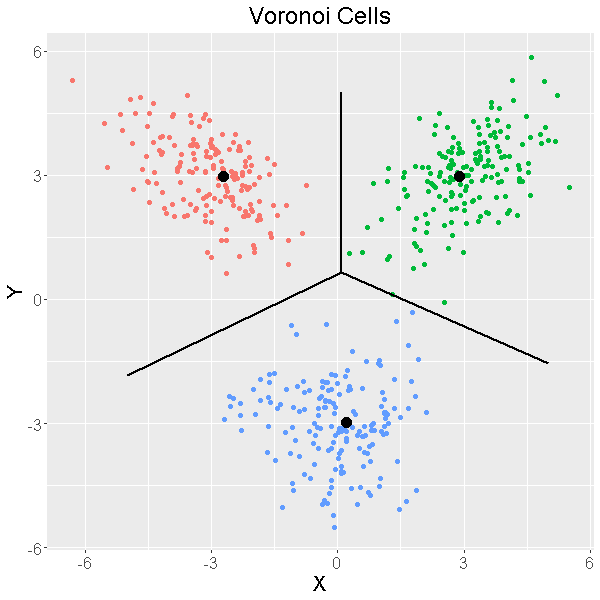}}
\end{tabular}
\caption{\footnotesize One shot experiments for Scenario~1. (a) Simulated data and protoclusters (Step 1). (b) Estimated (dashed) vs true (solid) marginal densities. (c) Estimated (dashed) vs true (solid) conquering functions. (d) Voronoi cells of the conquerors for $u=0.1$ (Steps 2 and 3).\label{num_1}}
\end{figure}

\section{Numerical Experiments on Artificial Data}\label{numerical}
\subsection{Simulation Setup and One Shot Experiments}\label{1shot}
In this section we study the performance of the proposed methods via numerical experiments. An exhaustive Monte Carlo simulation study will be presented in Section~\ref{montecarlo}. The Monte Carlo simulation will assess two data-driven approaches for setting the sieve size based on the conquering function as well as the strategy of setting a fixed low sieve size (e.g., $u = 0.1$). The data-driven approaches based on the conquering function will be called throughout as the \textit{plateau} ($u$ at which the longest plateau of $C(u)$ ends) and the \textit{edge} ($u$ at which the largest jump on $C(u)$ occurs); see Appendix~B for technical details.

Pitfalls 1--2 from Section~\ref{introduction} motivated us to design the following simulation scenarios:
\begin{itemize}
\item \textbf{Scenario 1}: Data are drawn from a mixture of $\mathcal{K}=3$ bivariate Normal distributions with weights, mean vectors, and covariance matrices as in Example~\ref{mix3}. To study the clustering performance as the sample size increases, we consider $n\in \{50,100,250,500,1000\}$. In Figure~\ref{num_1} we depict a one shot example of Reign-and-Conquer algorithm corresponding to a sample of size $n=500$ and $u=0.1$; as can be seen in Figure~\ref{num_1} (d), the proposed method suitably partitions the multivariate data. 
\item \textbf{Scenario 2}: Data are drawn from a mixture of ${\mathcal{K}}=3$ Clayton copulas \cite[Chapter 4.2]{nelsen2007introduction} with Normal margins: $f_X(x)= \phi(x; -5, 4^2)/2+ \phi(x; 3,4^2)/2$, and $f_Y(y)= \phi(x;-5,1)/3+ \phi(x; 2.5,1)/3+\phi(x;5,1)/3$. Trivially, the joint distribution does not obey \eqref{mvn}, and the number of clusters per margin is different ($K = 2$, $L = 3$). Here, we also consider sample sizes $n\in \{50,100,250,500,1000\}$, and in Figure~\ref{num_2} we depict the outcome of a one shot experiment with $n=500$. As can be seen in Figure~\ref{num_2} (d), the proposed method suitably partitions the multivariate data. 
\item \textbf{Scenario 3}: Data are drawn from a mixture of $d$-variate Normal distributions in dimensions $d\in \{5,10,15,20\}$ (moderate high dimensional data). 
In this scenario, the sample sizes and the number of clusters depend on $d$ in the following way: $n_d\equiv n  = \lfloor 10d^{3/2}\rfloor$ and ${\mathcal{K}}_d \equiv {\mathcal{K}} =\texttt{round}(\sqrt{d+1})$, where $\lfloor \cdot \rfloor$ and \texttt{round}() denotes the the floor and  round functions respectively. The covariance matrices and mixing probabilities are specified as $\boldsymbol \Sigma_k = \mathbf{I}_d$ and $\pi_k = 1/d$, for $k=1,\dots,{{\mathcal{K}}}$, whereas the mean vectors $\mu_k=(\mu_{1}^{(k)},\dots,\mu_{d}^{(k)})$ are sparsely defined: $\mu_{i}^{(k)}=0$ for $i\neq k$ and $\mu_{k}^{(k)}=d/\sqrt{2}$. This scenario leads to several identical marginal distributions, and the mean vector components are constrained to be equidistant $\|\boldsymbol \mu_i - \boldsymbol \mu_j\|_2 = d$ for all $j\neq i$; therefore the separation between clusters increases linearly with the number of dimensions. 
\end{itemize}

\begin{figure}[h!]
  \centering
  \footnotesize \textbf{Scenario~2}\\
  \begin{tabular}{cc}
\subfloat[]{\includegraphics[scale=0.3]{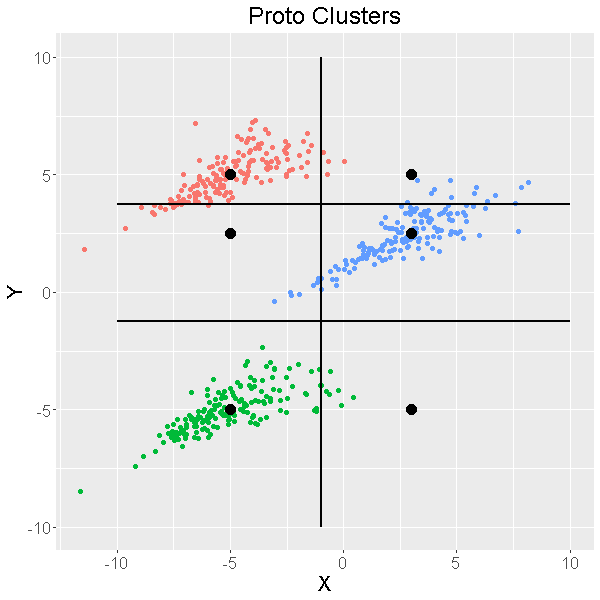}
} &
\subfloat[]{\includegraphics[scale=0.3]{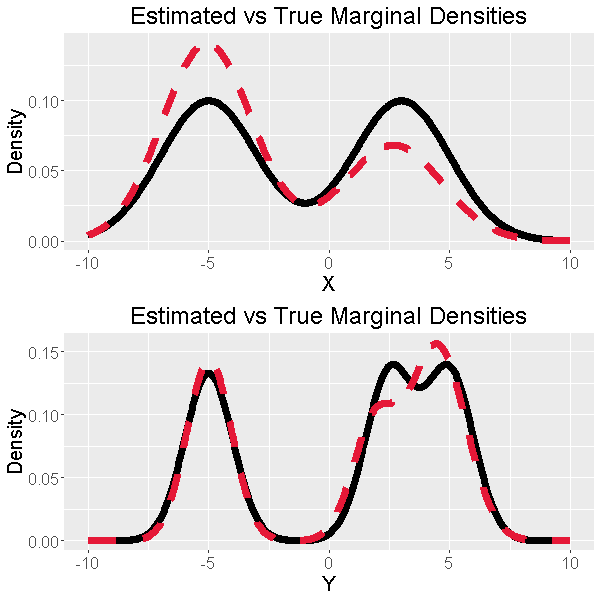}} 
\end{tabular}
\centering \\
\centering
  \begin{tabular}{cc}
\subfloat[]{\includegraphics[scale=0.3]{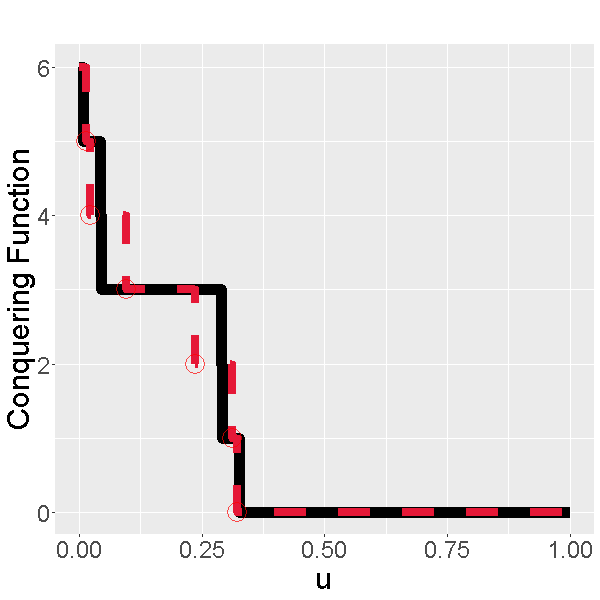}
} &
\subfloat[]{\includegraphics[scale=0.3]{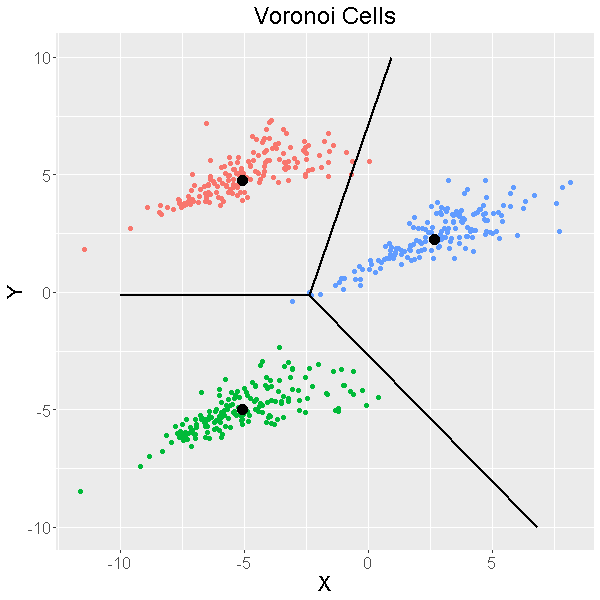}}
\end{tabular}
\caption{\footnotesize One shot experiments for Scenario~2. (a) Simulated data and protocluster (Step 1). (b) Estimated (dashed) vs true (solid) marginal densities. (c) Estimated (dashed) vs true (solid) conquering functions. (d) Voronoi cells of the conquerors for $u=0.1$ (Steps 2 and 3).\label{num_2}}
\end{figure}

\subsection{Monte Carlo Simulation Study}\label{montecarlo}
To assess the performance of the proposed clustering approach, for Scenarios 1--3 we redo the previous one shot analysis $M=1000$ times so to estimate the following clustering agreement metrics: Rand and Adjusted Rand Index (RI and ARI respectively), Jaccard Index (JI), and Fowlkes--Mallows Index (FMI); see \cite{pfitzner2009characterization} and references therein. We also report the empirical distribution of the number of clusters detected by the proposed method, over different conquering strategies (i.e., fixed $u$, plateau, and edge), along with the same outputs for plain vanilla GMM (Gaussian Mixture Model)-based clustering. For GMM, model selection was conducted using the BIC, and the model was fitted using the \texttt{mclust} \cite{fraley2012package} package in \texttt{R}.

\begin{figure}[h]
\centering
\textbf{Scenario 1}\\ \vspace{0.2cm}
\includegraphics[scale=0.3]{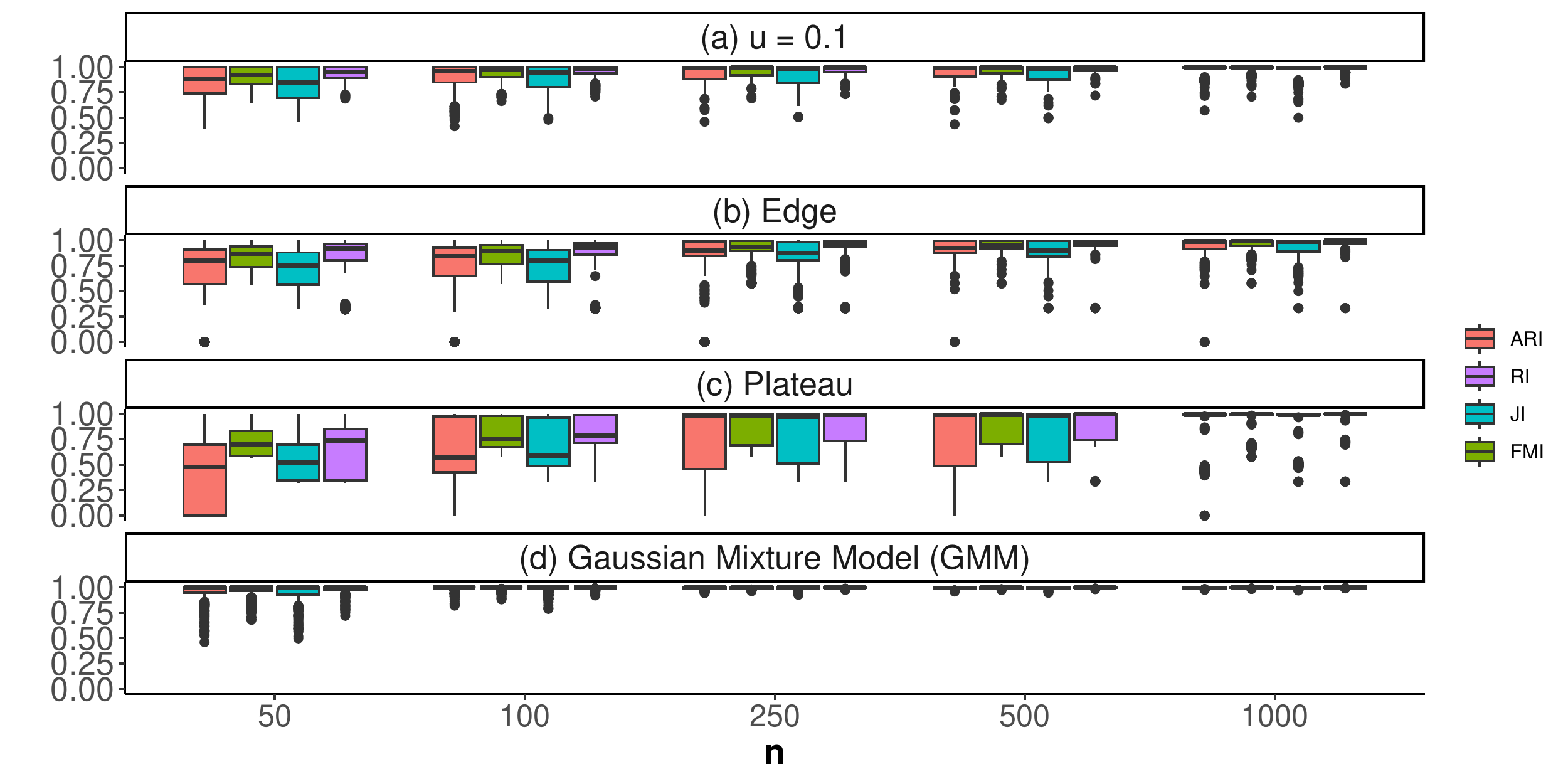}
\includegraphics[scale=0.3]{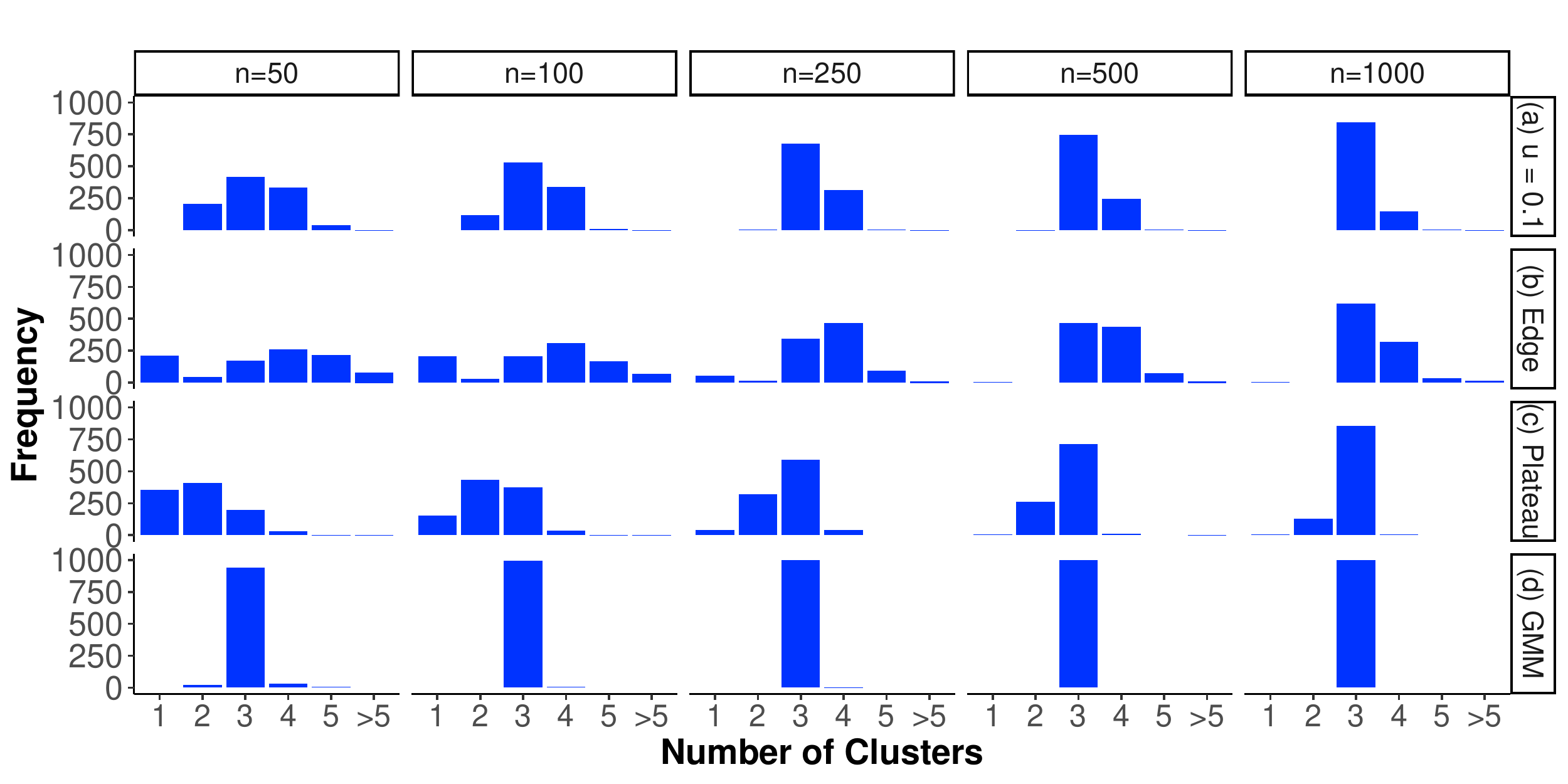} 
\textbf{Scenario 2}\\ \vspace{0.2cm}
\centering
\includegraphics[scale=0.3]{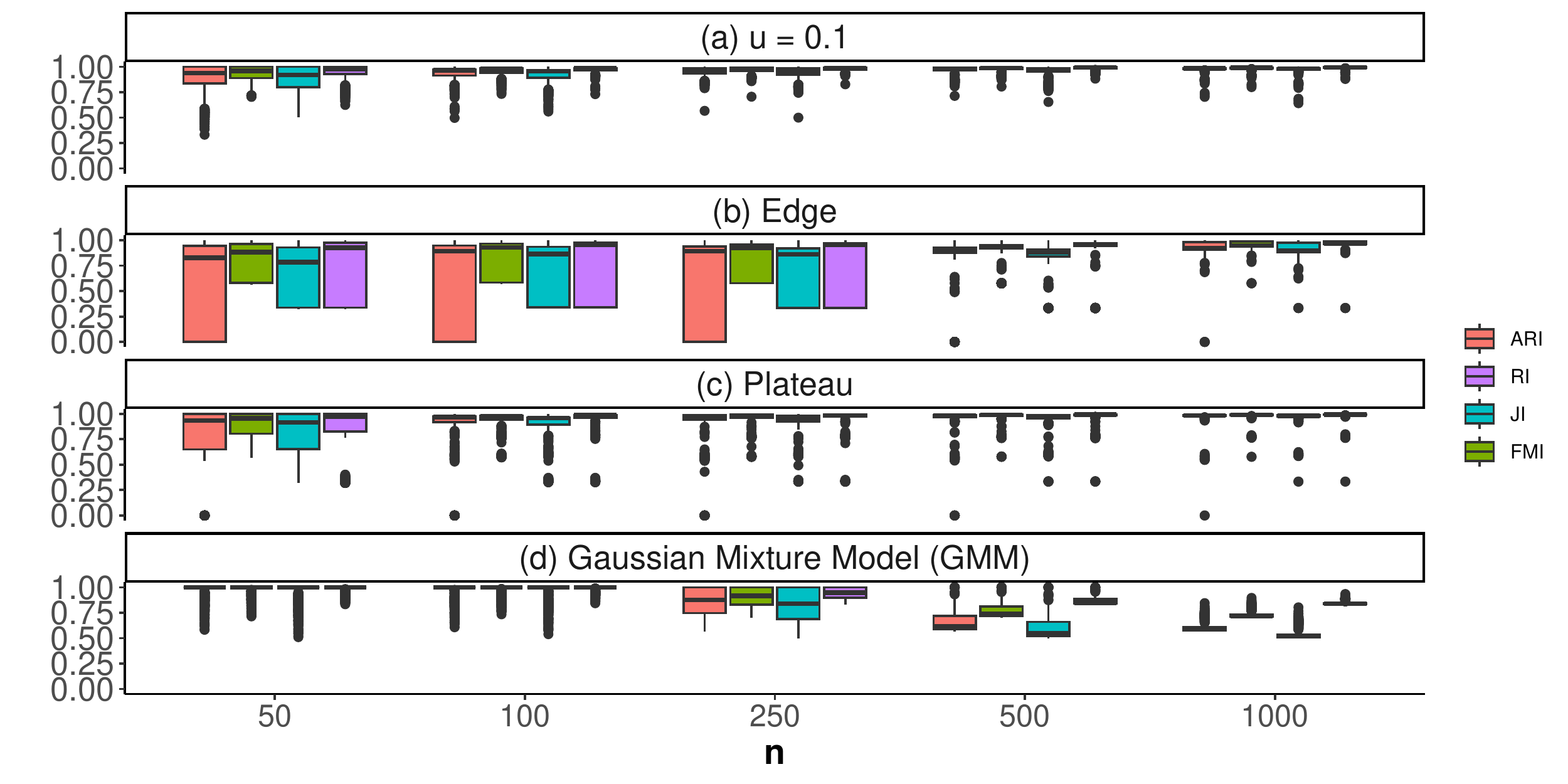}
\includegraphics[scale=0.3]{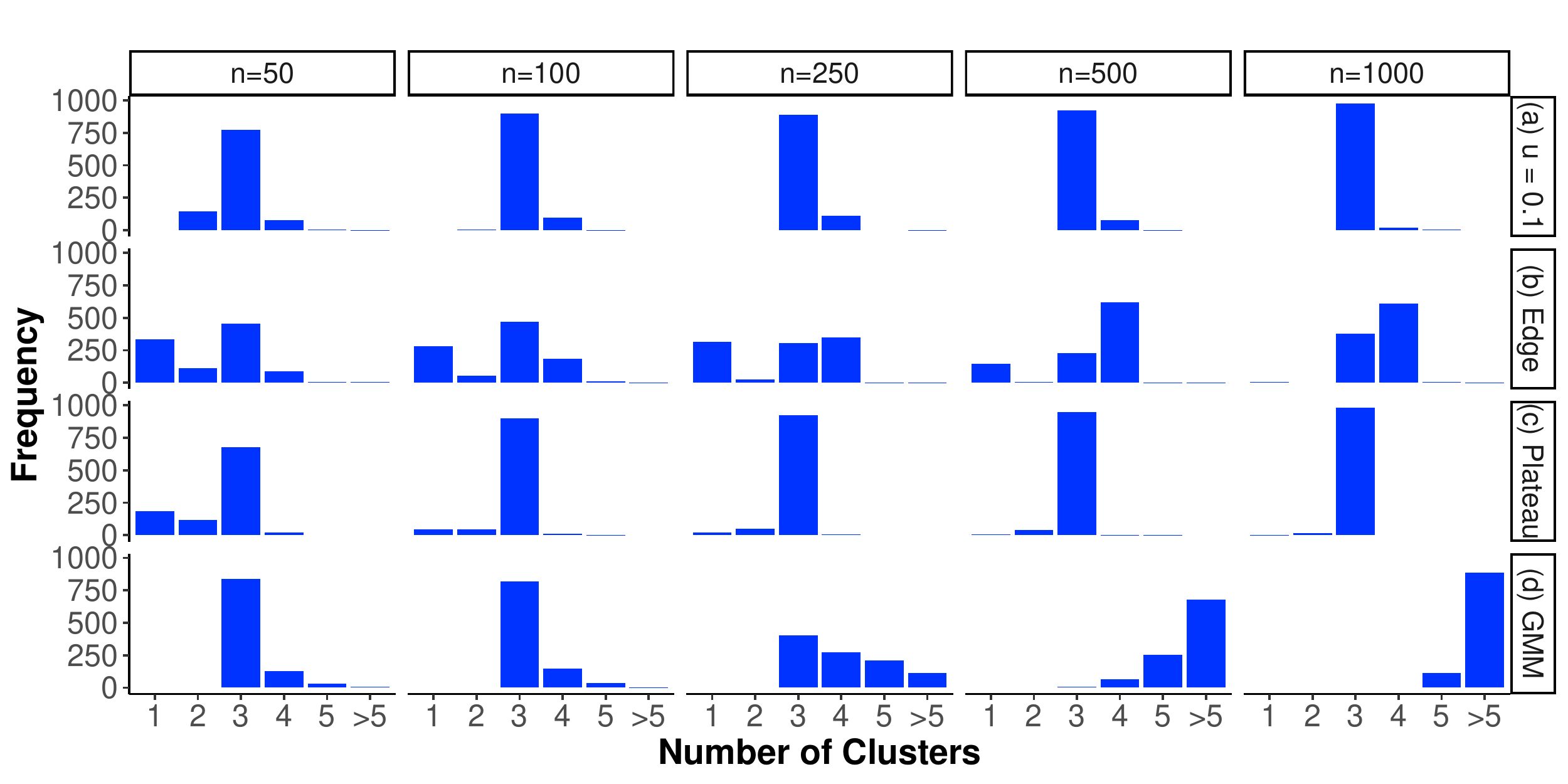}
\textbf{Scenario 3}\\ 
\includegraphics[scale=0.3]{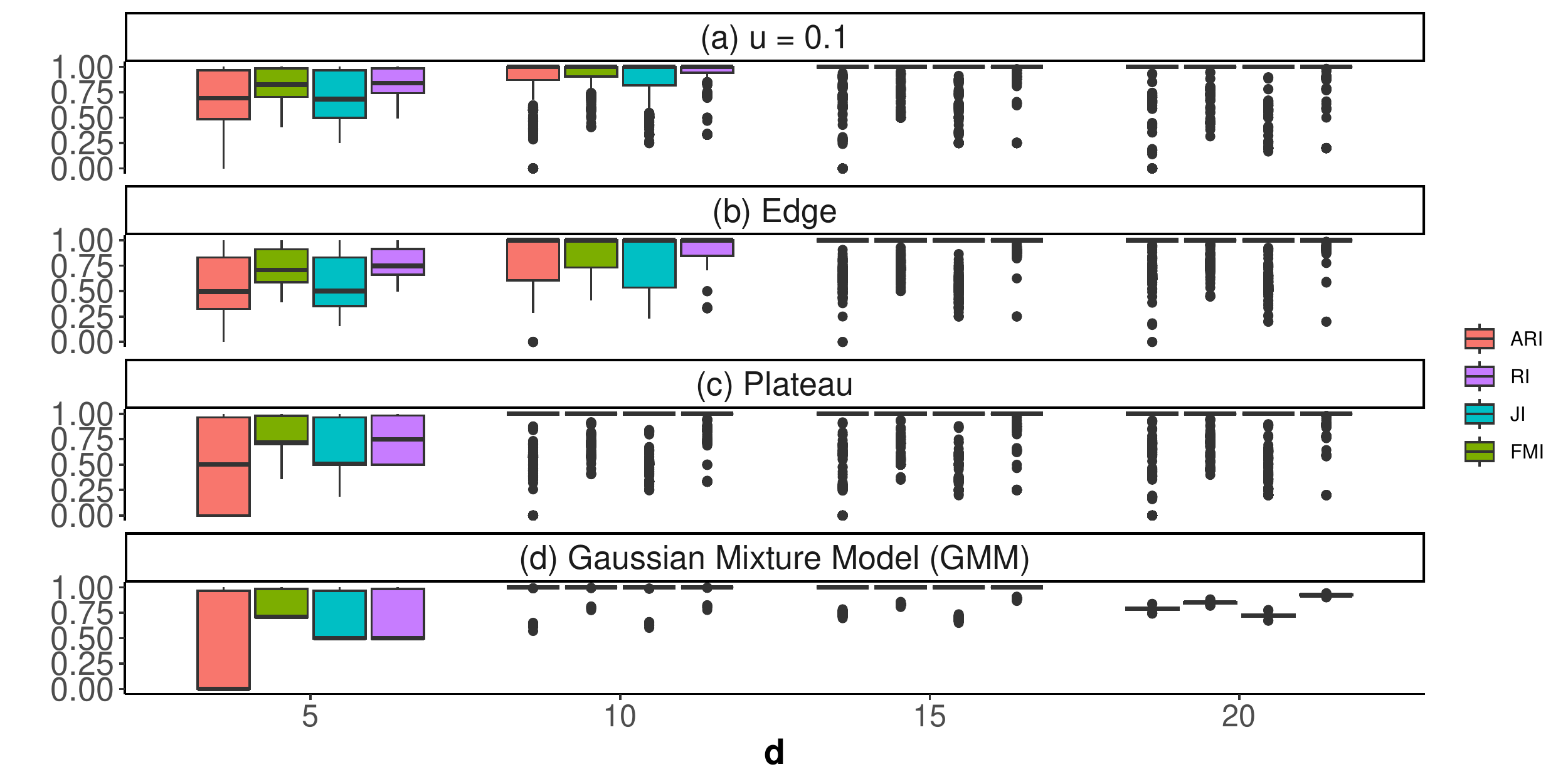}
\includegraphics[scale=0.3]{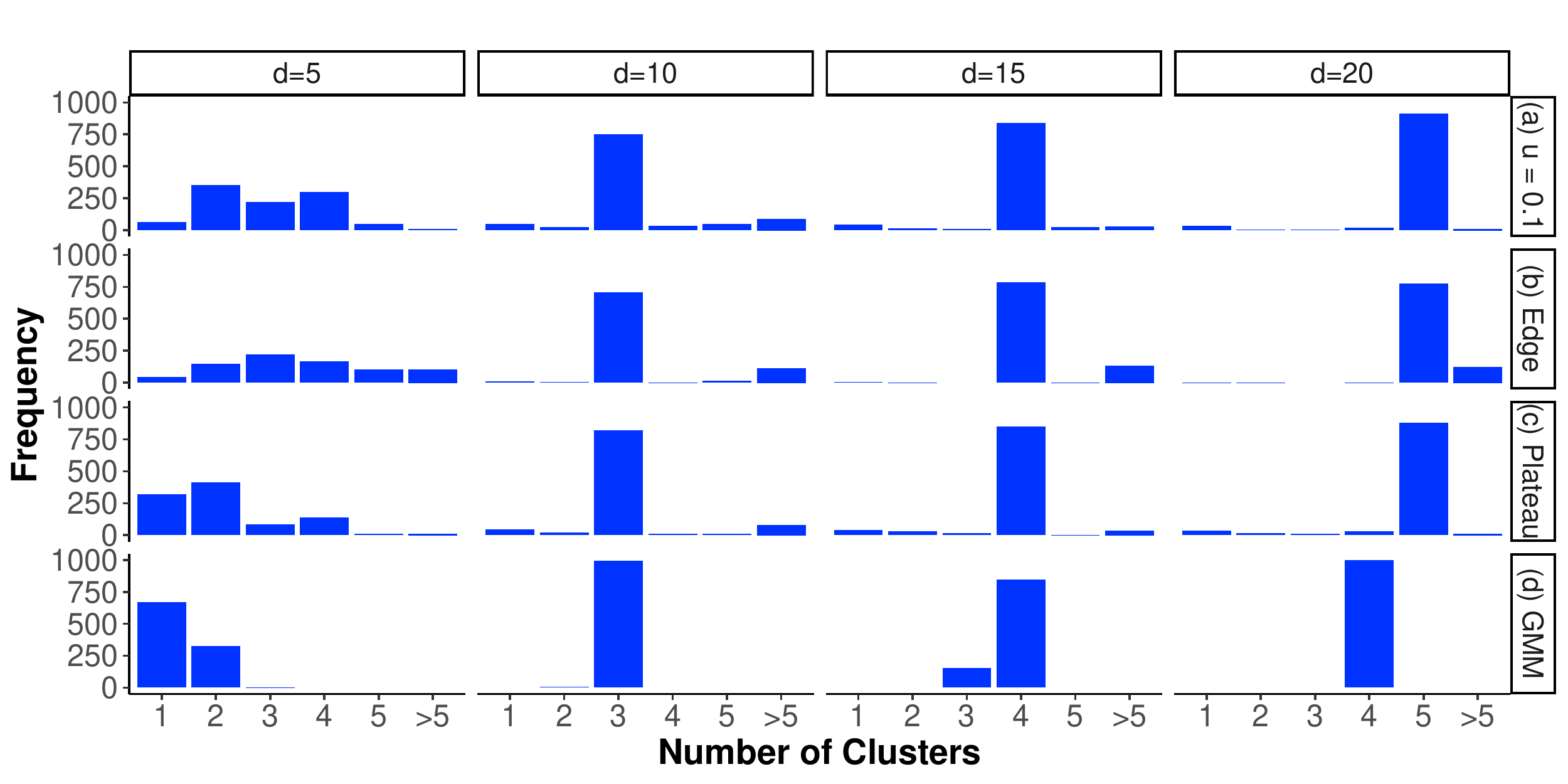}
\caption{\footnotesize Monte Carlo simulation study: (Left) Performance metrics (ARI, RI, JI, FMI) (Right) Empirical distribution on the number of detected clusters.\label{monte}}
\end{figure}

In Figure~\ref{monte} (top and middle), it can be seen that for Scenarios~1--2, as the sample size increases the performance metrics increase on average for all conquering strategies. Interestingly, for Scenario 1, the edge conquering strategy works better than the plateau for relatively small sample sizes. Conversely, in Scenario 2, the plateau yields on average better results than the edge for small sample sizes. In Scenarios 1--2, fixing a sieve size of $u=0.1$, produces accurate clustering results on average even for small sample sizes. In addition, as the sample sizes increases the proposed method identifies most frequently the correct number of clusters ${{\mathcal{K}}}=3$ for Scenarios~1--2---both when $u=0.1$ as well as when $u$ is set using the plateau.

In Figure~\ref{monte} (bottom) we present the performance of the proposed method for Scenario 3. As can be seen in Figure~\ref{monte} (bottom--left), as the dimension increases on average the proposed method presents better agreement metrics (recall that in Scenario~3 cluster separation grows linearly with data dimension). In Figure~\ref{monte} (bottom--right), it can be seen that as the sample size increases, the proposed method most frequently captures the true number of clusters $\mathcal{K}_d \in \{2,3,4,5\}$ for dimensions $d\in\{5,10,15,20\}$ respectively.

Some final comments on the comparison of the Reign-and-Conquer clustering against GMM are in order. In Scenario~1 the data are from are simulated from a low-dimensional Gaussian mixture model, and hence perhaps not surprisingly GMM overperforms the proposed approach. Still, the performance of the proposed approach is still remarkable especially as we make no assumption on the joint. In addition, Reign-and-Conquer has a comparable, if not superior, performance than GMM over Scenarios 2--3. 

\begin{figure}[h!]
\begin{center}
\textbf{Banknotes}
\end{center}
\centering
  \begin{tabular}{cc}
\subfloat[]{\includegraphics[scale=0.33]{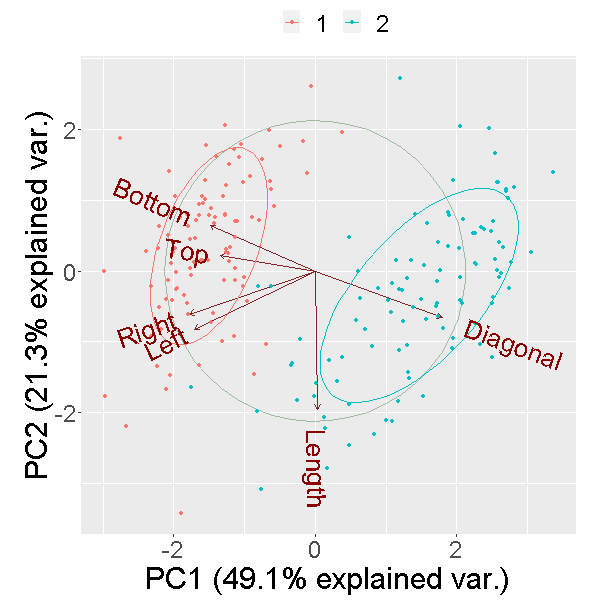}
} &
\subfloat[]{\includegraphics[scale=0.33]{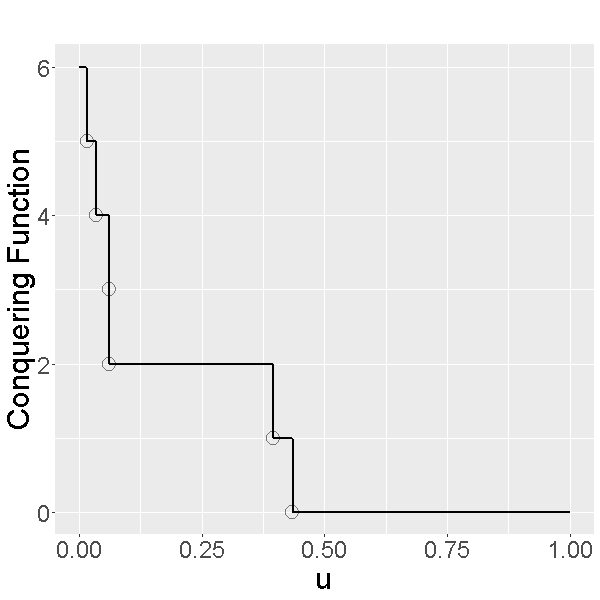}} 
\end{tabular}\\
\begin{center}
\textbf{Italian Wine}
\end{center}
\centering
  \begin{tabular}{cc}
\subfloat[]{\includegraphics[scale=0.33]{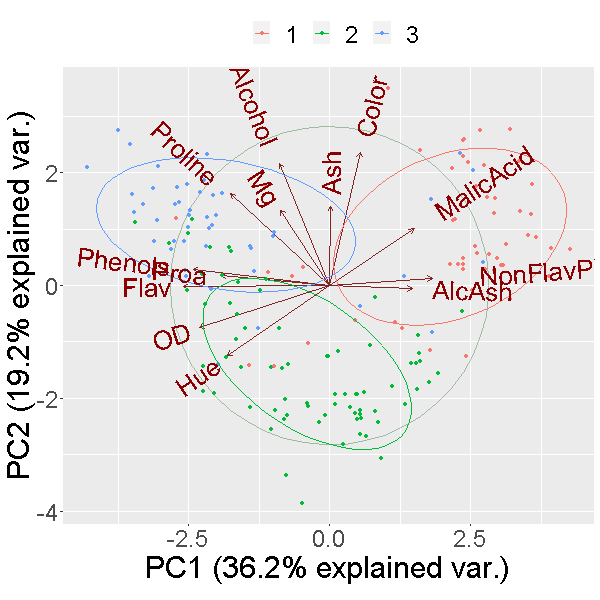}
} &
\subfloat[]{\includegraphics[scale=0.33]{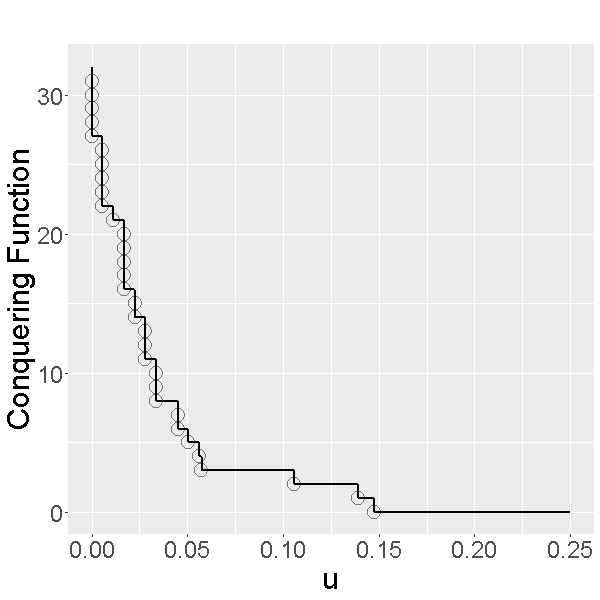}} 
\end{tabular}
\caption{\footnotesize (a) Biplot banknote data set, the color key corresponds to the cluster labels obtained for $u=0.36$ (Plateau). (b) Conquering function estimate. (c) Biplot of wine data set, the color key corresponds to the cluster labels obtained for $u=0.105$ (Plateau). (d) Conquering function estimate (restriced to $[0,0.25]$ for visualization purposes).\label{app_2}}
\end{figure}

\section{Real Data Illustrations}\label{applications} 
\subsection*{Banknote Data}
The first dataset to be analyzed with the proposed methods contains $p=6$ measurements made on 100 genuine and 100 counterfeit old-Swiss 1000-franc bank notes \cite[][pp.~5--8]{flury1988multivariate}. The data consists of the following measurements (in millimetres): length of bill, width of left edge, width of right edge, bottom margin width, top margin width, and the width of diagonal. The data are available from the \texttt{mclust}  \texttt{R} package \citep{fraley2012package}. In Figure~\ref{app_2} (a) we depict a biplot to represent the first two principal components of the data, along with the corresponding clustering yield by the proposed Reign-and-Conquer clustering. The sieve size was set using the Plateau conquering strategy that corresponds to the best average results on the simulations setting in Scenario~3. In Figure~\ref{app_2} (b) we depict the fitted conquering function, from where it can be seen that the plateau consists of $u = 0.36$. As can be noticed from the confusion matrix in Table~\ref{tab:comp1}, Reign-and-Conquer does an excellent job classifying counterfeit as well as genuine data. The GMM analysis combined with BIC suggests that there could be three clusters. The latter analysis is not least interesting from a forensic viewpoint, as it suggests that there might be two clusters of counterfeit banknotes. Finally, we note that the number of marginal clusters obtained using Reign-and-Conquer (i.e., $K_1, \dots, K_6$) ranges from 1--3 clusters. 

\subsection*{Italian Wine Data}
\noindent The second dataset on which the proposed approach will be illustrated contains the results of a chemical analysis of wines grown in the same region of Italy derived from three cultivars (Barbera, Grignolino, and Barolo); the data are available from \cite{vandeginste1990parvus}. The chemical analysis includes the measurement of $d=13$ continuous variables (such as Alcohol, Malic acid, Ash, Flavanoids, etc) on $n=178$ instances. Similarly to the banknote data illustration, in Figure~\ref{app_2} (c) we depict a biplot to represent the first two principal components of the data, along with the corresponding clustering yield by the proposed method. The sieve size was set using the Plateau conquering strategy that corresponds to the best average results on the simulations setting in Scenario~3. In Figure~\ref{app_2} (d) we depict the fitted conquering function, the plateau consists of $u = 0.105$. As can be seen from the confusion matrix in Table~\ref{tab:comp1}, Reign-and-Conquer learns about the `right' number of cultivars and the obtained clusters have a resemblance with the cultivars. The GMM analysis combined with BIC would offer another interesting outlook, suggesting that two of these clusters are so similar that they should perhaps be merged. Finally, we note that the number of marginal clusters (i.e., $K_1, \dots, K_{13}$) is 1 for 8 of the dimensions, and 2 for the remainder dimensions.\\


\begin{table}\footnotesize 
\parbox{.45\linewidth}{
\centering
\begin{tabular}{|l|c|cc|}
\hline \hline 
\textbf{Method} & \textbf{Cluster} &\textbf{Counterfeit} &\textbf{Genuine}\\
\hline
\multirow{2}*{RC} & 1& 99 &0\\
 & 2& 1&100\\
\hline
\hline
\multirow{3}*{GMM} & 1& 16 &2\\
 & 2& 0&98\\
 & 3& 84&0\\
\hline \hline 
\end{tabular} 
}
\hspace{.3cm} 
\parbox{.45\linewidth}{
\centering
\begin{tabular}{|l|c|ccc|} 
\hline \hline 
\textbf{Method} & \textbf{Cluster} &\textbf{Barbera}&\textbf{Grignolino} &\textbf{Barolo}\\
\hline
\multirow{3}*{RC} & 1& 42 &9 & 5\\
 & 2& 1&56&16\\
& 3& 5&7&37\\
\hline
\hline
\multirow{2}*{GMM} & 1& 0&27 &59\\
 & 2& 48 &44 & 0\\
\hline \hline
\end{tabular}
}
\caption{\footnotesize Confusion matrix for banknote (left) and wine (right) data. \textsc{Note}: RC = Reign-and-Conquer, GMM = Gaussian Mixture Model. \label{tab:comp1}}
\end{table}

\section{Final Observations and Concluding Remarks}
\label{closing}
This paper devises an unsupervised learning approach based on a marginal model-based specification followed-up by a strategy-game inspired algorithm that partitions the sample space. The approach was motivated from Pitfalls~1--2 from Section~1, and it can be used for clustering data---both in a multivariate manner as well as marginally. Pitfall~1---the single $K$ problem---implies that fitting a plain vanilla multivariate Gaussian model might result in all margins and the joint having the same number of components, unless one is able to penalize for deviations from $\mu_{k, j} = \mu_{k', j}$ and $\sigma_{k, j} = \sigma_{k', j}$. Pitfall~2---the curse of dimensionality---is well known, and it implies that learning about a plain vanilla Gaussian mixture model implies  learning about $O(d^2)$ parameters, when $d \to \infty$. Prompted by these concerns, the herein proposed clustering approach only specifies a model for the margins but leaves the joint unspecified, it has the advantage of being parallelizable, and bypasses the need to learn about $\mathcal{K}d(d+1)/2$ parameters used in the covariance matrices $\boldsymbol\Sigma_1, \dots, \boldsymbol\Sigma_{\mathcal{K}}$ required for a `full' (joint) Gaussian model-based clustering approach. The conducted numerical experiments suggest that the proposed approach has a comparable performance, and even in some cases superior, than a plain vanilla Gaussian model-based approach. 

While the obtained numerical evidence indicates a satisfactory performance of the proposed method under a variety of situations, there is still room for improvement, open problems to be addressed as well as opportunities for future research. First, the geometry of the boundaries of the final kingdoms (i.e., the Voronoi Cells of the conquerors) could perhaps be bended so to better adapt to the structure of the data, to offer more flexibility to the resulting partitions, and ultimately to improve clustering. Second, the game-theoretical variant from Section~\ref{gtheory} opens a world of opportunities on ways to set the `incentives' to conquer, via an utility function, to explored in a follow-up paper. In terms of the ``single $K$ problem'', an alternative to the path taken here would be to develop Bayesian regularization approaches that aim to penalize for deviations from $\mu_{k, j} = \mu_{k', j}$ and $\sigma_{k, j} = \sigma_{k', j}$. Finally, while here the focus has been on unsupervised learning, the potential of related strategy-game inspired approaches for supervised learning would seem natural. 

\noindent \textbf{Acknowledgments}: We thank participants of IFCS 2022 for insightful  comments, discussions, and feedback.\\ 

\noindent \textbf{Funding}: MdC was partially supported by FCT (\textit{Fundação para a Ciência e a Tecnologia}, Portugal) through the project and UID/MAT/00006/2020. \vspace{0.2cm}
\section*{{Appendix}}\footnotesize 
\subsection*{{Appendix A: Proofs of Theoretical Results}}\footnotesize 
\label{Proofs}



  Before getting started with the proofs we lay the groundwork. The proof of Theorem~\ref{propsC} uses the following representation of the conquering function
  \begin{equation}\label{altdef}
    C(u) = \prod_{j = 1}^{d} K_j - \sum_{\mathbf{i} \in I} 1_{D_u}(\mathbf{i}), 
  \end{equation}
  which follows directly from \eqref{conq2}. Here, $1_A$ is the indicator of set $A$ and in the proof we will make use of some of its well-known properties \citep[e.g.,][Chapter~1]{resnick2019}, such as 
  \begin{equation}\label{suminf}
    \underset{n\to \infty}{\lim \sup}~1_{A_n} = 1_{{\lim \sup}_{n\to \infty}~A_n}, \qquad
    \underset{n\to \infty}{\lim \inf}~1_{A_n} = 1_{\lim \inf_{n \to \infty}~A_n}.
  \end{equation}
  Since \eqref{suminf} holds for both $\lim \sup$ and $\lim \inf$ it follows that $\lim_{n \to \infty} 1_{A_n} = 1_{\lim_{n \to \infty} A_n}$. Recall in addition that if $\{A_n\}$ is an nondecreasing sequence of sets, then its limit is the infinite union, that is
  \begin{equation}\label{lim}
    A_n \subseteq A_{n + 1} \Longrightarrow \lim_{n \to \infty} P(A_n) = \bigcup_{n = 1}^{\infty}\,A_n. 
  \end{equation}
  See, for instance, \citet[][Proposition~1.4.1]{resnick2019}. Finally, the proof of Claim~d) in Theorem~\ref{propsC} will make use of the Lebesgue measure over the unit interval, $\lambda([a, b]) = b - a$, for $[a, b] \subseteq [0, 1]$. \\
\noindent  \textit{Proof of Theorem~\ref{propsC}.}
  \begin{enumerate}[a)]
  \item Consider $(u, v) \in [0, 1]^2$ such that $u \leq v$. Then, whenever $P(A_{\mathbf{i}}) \leq u$ it follows that $P(A_{\mathbf{i}}) \leq v$; or in other words $D_{u} \subseteq D_{v}$, which in turn implies that $1_{D_{u}}(\mathbf{i}) \leq 1_{D_{v}}(\mathbf{i})$. Hence, 
    \begin{equation*}
      \begin{split}
        - \sum_{\mathbf{i} \in I} 1_{D_{u}}(\mathbf{i}) \geq - \sum_{\mathbf{i} \in I}1_{D_{v}}(\mathbf{i}) &\Longrightarrow
        \prod_{j = 1}^{d} K_j - \sum_{\mathbf{i} \in I} 1_{D_{u}}(\mathbf{i}) \geq
        \prod_{j = 1}^{d} K_j - \sum_{\mathbf{i} \in I} 1_{D_{v}}(\mathbf{i}) \\
        &\Longrightarrow C(u) \geq C(v),
      \end{split}
    \end{equation*}
    from where the final result follows. \strut \hfill \qed 
  \item First note that the proof of Claim~a) implies that $D_u$ is a nondecreasing, in the sense $D_u \subseteq D_v$, for any $u \leq v$ with $(u, v) \in [0, 1]^2$. Next, consider an arbitrary $u \in [0, 1]$ and a sequence 
    $u_n$ such that $u_n \to u$, with $u_n \leq u$ for every $n \in \mathbb{N}$. Then, for a sufficiently large $n$ it holds that $u_n \leq u_{n + 1}$ which in turn implies that $D_{u_n} \subseteq D_{u_{n + 1}}$. This, along with \eqref{lim} and the fact that $D_u$ is nondecreasing, implies that $\lim_{n\to \infty} D_{u_n} = \bigcup_{n = 1}^\infty D_{u_n} = D_u$. Finally, \eqref{altdef} and \eqref{suminf} then yield that 
    \begin{equation*}
      \lim_{n \to \infty}\,C(u_n) = \prod_{j = 1}^d K_j - \sum_{\mathbf{i} \in I} \lim_{n\to\infty}\,1_{D_n}(\mathbf{i)} =
      \prod_{j = 1}^d K_j - \sum_{\mathbf{i} \in I} \,1_{\lim_{n\to\infty} D_n(\mathbf{i})} =
      \prod_{j = 1}^d K_j - \sum_{\mathbf{i} \in I} 1_{D_u(\mathbf{i})} = C(u),
    \end{equation*}
    which concludes the proof. \strut \hfill \qed 
  \item Trivially, since Claim~a) shows that $C(u)$ is nonincreasing it follows that for every $u \in [0, 1]$, 
    \begin{equation*}
      C(u) \geq C(1) = \prod_{j = 1}^d K_j - |D_1| = 0,
    \end{equation*}
    where the final equality is a consequence of the fact that $|D_1| =|\{\mathbf{i}:P(A_{\mathbf{i}})\leq1\}|=|\{A_{\mathbf{i}}:\mathbf{i} \in I\}| = \prod_{j = 1}^d K_j$. The final result then follows from \eqref{altdef} and \eqref{suminf} by noting that for every $u \in [0, 1]$, 
    \begin{equation*}
      \begin{split}
        C(u) \leq C(0) = 
      \prod_{j = 1}^d K_j - |D_0| = 
      \prod_{j = 1}^d K_j,
      \end{split}
    \end{equation*}
    since $|D_0| = |\{\mathbf{i}:P(A_{i,j})\leq 0\}| = |\emptyset| = 0$. \strut \hfill \qed
  \item First note that, 
    \begin{equation}\label{trivia}
      1_{D_u^c}(\mathbf{i}) =
      \begin{cases}
        1, & \mathbf{i} \in D_u^c, \\
        0, & \mathbf{i} \in \in D_u,
      \end{cases} ~=~
      \begin{cases}
        1, & 0 \leq u < P(A_{\mathbf{i}}), \\
        0, & \text{otherwise}.
      \end{cases}
    \end{equation}
    Next, observe that $|D_u^c| = \sum_{\mathbf{i} \in I} 1_{D_u^c}(\mathbf{i})$ which along with \eqref{trivia} yields  
    \begin{equation*}
      \begin{split}
        \int_0^1 C(u) \, \dif u = \int_0^1 |D_u^c| \, \dif u  
         = \sum_{\mathbf{i} \in I} \int_0^1 1_{D_u^c}(\mathbf{i}) \, \dif u 
         = \sum_{\mathbf{i} \in I} \int_{[0, P(A_\mathbf{i})]}  \, \dif u 
         = \sum_{\mathbf{i} \in I}\lambda ([0, P(A_\mathbf{i})]) 
        = \sum_{\mathbf{i} \in I} P(A_{\mathbf{i}})=1,     
      \end{split}
    \end{equation*}
    which concludes the proof.
\strut \hfill \qed
\end{enumerate}
\subsection*{{Appendix B: Step Function Representation, Plateau, and Edge}}\footnotesize
This appendix shows formally that the conquering function is a step function with a finite number of steps (provided that $K_1, \dots, K_d$ are finite), and it uses that representation so to formally define the plateau and the edge. As a consequence of \eqref{altdef} and of \eqref{trivia} in Appendix~A it holds that 
\begin{equation}\label{steprep}
  \begin{split}
    C(u) = \prod_{j = 1}^{d} K_j - \sum_{\mathbf{i} \in I} 1_{D_u}(\mathbf{i}) 
    = \prod_{j = 1}^{d} K_j - \sum_{\mathbf{i} \in I} 1_{[0, P(A_{\mathbf{i}}))}(u) 
    = \sum_{\mathbf{i} \in I} \{1 - 1_{[0, P(A_{\mathbf{i}}))}(u)\} 
    = \sum_{\mathbf{i} \in I} 1_{[P(A_{\mathbf{i}}), 1]}(u), 
  \end{split}
\end{equation}
where the final equality follows from the well-known property of the indicator, $1 - 1_B = 1_{B^c}$. Hence, Equation~\eqref{steprep} shows that $C(u)$ is a step function with a maximum of $|I| = \prod_{j = 1}^{d} K_j$ steps. Given this representation, it follows that 
\begin{equation*}
  \text{plateau} = \sup\bigg\{u: C(u) = \underset{\mathbf{i} \in I}{\max}[P(A_{\mathbf{i}}^c)]\bigg\}, \qquad
  \text{edge} = \arg \max_{u} \,\{C(u) - C(u^+)\}.
\end{equation*}
In words, the plateau is the value at which the longest plateau of $C(u)$ ends, and the edge is the value at which the largest jump on $C(u)$ occurs.

\noindent

\renewcommand\refname{{References}} 
\bibliographystyle{asa2}  
\bibliography{library.bib} 

\end{document}